\documentclass{article}

\usepackage[main,final]{neurips_2026}

\usepackage[utf8]{inputenc}
\usepackage[T1]{fontenc}
\usepackage{hyperref}
\hypersetup{hidelinks,
  pdftitle={The Reflexivity Threshold: A Phase Transition for Multi-Agent Performative Prediction},
  pdfauthor={Ahmed Mahrous and Roberto Di Pietro}}
\usepackage{url}
\usepackage{booktabs}
\usepackage{amsfonts}
\usepackage{amssymb}
\usepackage{amsmath}
\usepackage{amsthm}
\usepackage{nicefrac}
\usepackage{microtype}
\usepackage{xcolor}
\usepackage{mathtools}
\usepackage{subcaption}
\usepackage{enumitem}
\setlist{nosep,leftmargin=*}

\usepackage[font=small,labelfont=bf]{caption}
\makeatletter
\renewcommand\paragraph{\@startsection{paragraph}{4}{\z@}%
  {0.25ex plus 0.1ex minus 0.1ex}%
  {-0.8em}%
  {\normalfont\normalsize\bfseries}}
\makeatother

\makeatletter
\def\thm@space@setup{%
  \thm@preskip=4pt
  \thm@postskip=4pt
}
\makeatother

\newtheorem{theorem}{Theorem}
\newtheorem{proposition}[theorem]{Proposition}
\newtheorem{corollary}[theorem]{Corollary}

\newtheorem{definition}[theorem]{Definition}
\newtheorem{assumption}{Assumption}

\DeclareMathOperator*{\argmin}{arg\,min}

\title{The Reflexivity Threshold: A Phase Transition for Multi-Agent Performative Prediction}

\author{%
  Ahmed Mahrous\thanks{Corresponding author.}\qquad Roberto Di Pietro \\
  King Abdullah University of Science and Technology (KAUST) \\
  Saudi Arabia
}

\begin{document}

\maketitle

\begin{abstract}

Many learning systems are performative: once deployed, their predictions change the data distribution on which future models are trained. In multi-agent settings, this feedback is routed through a network of interacting learners, so stability depends not only on the strength of performativity but also on who affects whom. We study repeated retraining in multi-agent performative prediction and identify a spectral threshold that separates stable learning from worst-case learning obstruction. Prior work often collapses cross-agent feedback into a scalar contraction constant, missing this network structure. We instead introduce the reflexivity matrix $\Gamma$, which records how strongly one agent's deployment can move another agent's best response, defined from per-agent gradient sensitivities and curvature parameters. The threshold occurs when the spectral radius of this matrix crosses one. Below the threshold, repeated retraining converges geometrically to a unique performatively stable equilibrium, with cumulative excess loss bounded uniformly in the time horizon. Above the threshold, we construct noisy supercritical instances in which every local-information learner incurs expected cumulative excess loss growing linearly for horizons up to order $1/(\rho(\Gamma)-1)$. Together, these results characterize the threshold behavior through upper and lower bounds and recover the single-agent scalar threshold from prior work as a special case.

\end{abstract}

\section{Introduction}
\label{sec:intro}



Many learning systems, once deployed, change the environment that later generates their training data. A credit scoring model can change how applicants behave; an algorithmic trading strategy can move the price data on which future strategies are trained. This feedback effect is the focus of performative prediction. Existing theory largely studies the case of one learner facing the consequences of its own deployment. Many real systems, however, contain multiple learners whose deployments \emph{jointly} shape the environment.

The multi-agent setting is not simply a larger version of the single-agent problem; it introduces different difficulties. Feedback can travel across agents: one agent's deployment changes the data distribution, which changes another agent's retraining problem, which may itself affect the data distribution or other agents. Stability therefore depends not only on the overall strength of performativity, but also on how feedback is routed through the agent network.

We ask whether there is a general threshold separating stable multi-agent performative learning from unavoidable learning failure. When does repeated retraining converge to a stable equilibrium? 

We answer this question through a reflexivity matrix $\Gamma$ with entries $\Gamma_{ij}:=a_{ij}/\mu_i$, measuring how strongly agent $j$'s deployment can move agent $i$'s best response, after normalizing by the curvature of agent $i$'s objective. Thus $\Gamma$ records who affects whom, and how strongly, in the performative feedback network.

Our main result is a phase transition at $\rho(\Gamma)=1$, where $\rho(\Gamma)$ is the spectral radius of $\Gamma$. This quantity captures the largest asymptotic amplification factor of the feedback network. If $\rho(\Gamma)<1$, cross-agent feedback dissipates: repeated retraining is a contraction in a Perron--Frobenius weighted norm, converges geometrically to a unique performatively stable equilibrium, and incurs bounded cumulative excess loss. Above the threshold, this positive guarantee cannot be extended. We construct supercritical systems in which every decentralized local-information learner suffers cumulative loss that grows linearly with time, regardless of the retraining rule. Under sensitivity and curvature assumptions, the spectral radius of the reflexivity matrix gives a sharp boundary for general stability guarantees: below one, all systems satisfying our assumptions contract, while above one, there exist supercritical systems for which decentralized learning incurs linear cumulative loss.

\paragraph{Contributions.} 
First, we introduce the reflexivity matrix $\Gamma$, a nonnegative matrix whose entry $\Gamma_{ij}$ records how strongly agent $j$'s deployment can move agent $i$'s best response, as determined by per-agent gradient sensitivities and curvature constants. We prove an entrywise best-response sensitivity bound that reduces nonlinear retraining dynamics with general strongly convex losses to a spectral comparison through $\Gamma$, making Perron--Frobenius theory available. Second, we prove the subcritical side of the threshold: if $\rho(\Gamma)<1$, repeated retraining is a contraction in a Perron--Frobenius weighted norm, yielding a unique performatively stable equilibrium, geometric convergence from any initialization, and cumulative dynamic excess loss bounded uniformly in the time horizon. Third, we prove the supercritical side: for a rank-one family with $\rho(\Gamma)=1+\varepsilon$, every decentralized local-information learner incurs cumulative excess loss $\Omega(\Delta^2 T)$ over horizons $T=O(1/\varepsilon)$. Together, these results identify $\rho(\Gamma)=1$ as a boundary for general stability guarantees under primitive gradient-sensitivity and curvature assumptions, and recover the single-agent scalar threshold from prior work when $k=1$.


\section{Related Work}
\label{sec:related}

This paper draws on three lines of research. The first is the single-agent theory of performative prediction, which introduces the framework we extend. The second is the multi-agent theory of performative prediction and decision-dependent games, which is where our contribution is most directly positioned. The third is the combined tradition of spectral methods for network games and statistical lower bound techniques, from which we import two standard technical tools. 

\paragraph{Performative prediction.}
\citet{perdomo2020performative} introduced performative prediction in a single-agent setting, formalizing performative stability and convergence of repeated retraining under sensitivity conditions. Subsequent work largely retained this single-agent viewpoint while extending the framework to stochastic, causal, nonconvex, and sequential settings \citep{mendlerdunner2020stochastic,miller2021outside,brown2022performative,hardt2022performative,mofakhami2023performative,mandal2023performative}. Closest to our threshold result, \citet{anagnostides2026computational} established hardness and lower bounds at the single-agent scalar threshold. We depart from this line by studying decentralized multi-agent retraining, where performative feedback is cross-agent and the scalar threshold is replaced by the spectral condition $\rho(\Gamma)=1$.

\paragraph{Multi-agent performativity.} The multi-agent extension is the most direct predecessor of our work, and the literature divides into three strands.

The first establishes scalar contraction conditions for convergence. \citet{narang2023multiplayer} formalized $n$-player performative prediction with a scalar aggregate threshold $\rho < 1$, and analogous scalar conditions appear in cooperative consensus learning \citep{li2022multiagent}, decentralized games with coupled constraints \citep{yan2024decentralized}, and a Wasserstein-free variant \citep{zhong2025towards}. We depart from this strand on two axes. We replace the scalar aggregate with a matrix $\Gamma$ that preserves per-agent-pair feedback, and we pair the upper bound with a worst-case information-theoretic lower bound above the threshold.

The second develops spectral characterizations in special models. \citet{wang2023network} derived a necessary-and-sufficient spectral condition for a linear-quadratic mean-squared-error model on multiplex networks. Their spectral object is a weighted adjacency matrix with exogenous structure, and the condition is specific to the linear-quadratic model. Our $\Gamma$ is built from primitive per-agent quantities, applies to general strongly convex losses, and is paired with a lower bound.

The third identifies dynamical phase transitions. \citet{piliouras2023multi} proved a transition from stability to Li--Yorke chaos as update aggressiveness grows under scalar aggregate influence. Their threshold is dynamical and concerns the long-run behavior of specific update rules; ours is information-theoretic and applies to every decentralized local-information learner. \citet{farina2026stability} show that, for bounded losses, no-regret online learning yields approximate performative stability of a mixture over joint deployment histories, without regularity assumptions on the distribution maps. Our subcritical result guarantees geometric convergence of the last retraining iterate to a unique performatively stable equilibrium under curvature and sensitivity assumptions.

\paragraph{Spectral methods and lower bounds.} Two technical traditions supply tools used in our analysis without serving as competing frameworks. Spectral conditions for equilibrium uniqueness and stability are well established in network games \citep{ballester2006networks,bramoulle2014strategic,parise2019variational}, where the spectral radius governs payoff coupling: agent $j$'s action enters agent $i$'s utility directly. In our setting, it governs distributional feedback through $D(w)$, and $\Gamma$ is built from gradient sensitivities and curvature constants rather than payoff coefficients. Our lower bound uses the two-point method \citep{tsybakov2009introduction}; recent extensions to interactive \citep{chen2024assouad} and decision-dependent \citep{cutler2024stochastic} settings do not cover our performative learning setting. Our construction is tailored to multi-agent performative feedback over a finite horizon.

\section{Model}
\label{sec:model}

This section specifies the multi-agent performative prediction model used throughout the paper. We describe the agents, their objectives, the dynamics of repeated retraining, and the assumptions that underlie the reduction of Section~\ref{sec:reduction}. 

There are $k$ agents. Each agent $i \in \{1, \dots, k\}$ chooses a deployment $w_i$ from a feasible set $\mathcal{W}_i = \mathbb{R}^{d_i}$. The joint deployment profile is $w := (w_1, \dots, w_k) \in \mathcal{W} := \mathcal{W}_1 \times \cdots \times \mathcal{W}_k$.
The deployed profile induces a shared endogenous distribution $D(w)$ over observations $z \in \mathcal{Z}$. This distribution represents the data-generating process that all agents face after the profile $w$ has been deployed. The dependence of $D$ on $w$ captures the performative nature of the system: the environment that agents observe is shaped by their own collective actions.

Each agent $i$ has a loss function $\ell_i : \mathcal{W}_i \times \mathcal{Z} \to \mathbb{R}$ and aims to minimize its performative risk
\begin{equation}
\label{eq:perf-risk-def}
F_i(u_i; w)
\;:=\;
\mathbb{E}_{z \sim D(w)} \big[ \ell_i(u_i; z) \big].
\end{equation}
The two arguments of $F_i$ play distinct roles. The first, $u_i$, is the candidate action optimized by agent $i$. The second, $w$, is the deployed profile that determines the environment through $D(w)$. Retraining couples these roles by replacing the deployed profile with agents' optimized actions.

For each agent $i$ and each deployed profile $w$, the best response is
\begin{equation}
\label{eq:best-response-def}
B_i(w)
\;:=\;
\argmin_{u_i \in \mathcal{W}_i} F_i(u_i; w).
\end{equation}
Assumption~\ref{ass:sc} below ensures that this minimizer exists and is unique. The joint best-response map is $B(w) := \big( B_1(w), \dots, B_k(w) \big)$.
An equilibrium of the system is a profile that is consistent with every agent's best response.

\begin{definition}[Performatively stable equilibrium]
\label{def:pse}
A profile $w^\star \in \mathcal{W}$ is a \emph{performatively stable equilibrium} if $w^\star = B(w^\star)$.
\end{definition}

At a performatively stable equilibrium, each agent's deployed action is optimal for the distribution induced by the joint deployment. Existence of such an equilibrium is not assumed; it is derived in Section~\ref{sec:subcritical} as a consequence of the spectral condition on the reflexivity matrix.

We study the repeated retraining dynamics in which every agent simultaneously replaces its current deployment with its best response,
\begin{equation}
\label{eq:dynamics}
w_{t+1} = B(w_t), \qquad t = 0, 1, 2, \dots,
\end{equation}
with initialization $w_0 \in \mathcal{W}$. The central question of the paper is when the sequence $\{w_t\}$ converges to a performatively stable equilibrium, and when convergence fails.

\paragraph{Main assumptions.}
Three assumptions are used throughout the paper. The first two are local to individual agents and enter the reduction of Section~\ref{sec:reduction}. The third is a mild regularity condition on the feasible sets that allows application of Banach's fixed-point theorem.

\begin{assumption}[Strong convexity in own action]
\label{ass:sc}
For each agent $i$ and every fixed deployed profile $w \in \mathcal{W}$, the map $u_i \mapsto F_i(u_i; w)$ is $\mu_i$-strongly convex on $\mathcal{W}_i$, with $\mu_i > 0$.
\end{assumption}

Strong convexity makes $B_i(w)$ unique and quantifies agent $i$'s resistance to gradient perturbations. Larger $\mu_i$ means that the best response is less sensitive to changes in the data distribution.

\begin{assumption}[Gradient sensitivity to the deployed profile]
\label{ass:gs}
For each $i$, $F_i$ is continuously differentiable in $u_i$, and there exist nonnegative constants $a_{ij} \ge 0$ for all $i, j \in \{1, \dots, k\}$ such that for every agent $i$, every candidate action $u_i \in \mathcal{W}_i$, and every pair of deployed profiles $w, w' \in \mathcal{W}$,
\[
\| \nabla_{u_i} F_i(u_i; w) - \nabla_{u_i} F_i(u_i; w') \|
\;\le\;
\sum_{j=1}^k a_{ij} \, \| w_j - w_j' \|.
\]
\end{assumption}

Gradient sensitivity quantifies how a change in the deployed profile affects each agent's first-order condition. The coefficient $a_{ij}$ measures how much agent $j$'s deployment disturbs agent $i$'s gradient. If agent $j$ has no effect on agent $i$'s optimization problem, then $a_{ij} = 0$. Larger values of $a_{ij}$ indicate stronger cross-agent coupling. The diagonal coefficient $a_{ii}$ captures agent $i$'s own self-reflexive effect through the environment.

Assumptions~\ref{ass:sc} and~\ref{ass:gs} are the only two channels through which agents interact mathematically in this paper. Every downstream result is derived from these two assumptions and their combination in Section~\ref{sec:reduction}. The reflexivity matrix $\Gamma$ of Section~\ref{sec:reduction} is built from exactly the parameters $\{\mu_i\}$ and $\{a_{ij}\}$ introduced here.

\begin{assumption}[Unconstrained domains]
\label{ass:closed}
Each feasible set is $\mathcal{W}_i = \mathbb{R}^{d_i}$.
\end{assumption}

This makes $\mathcal{W}$ complete under the norms used below, enabling Banach's fixed-point theorem, and ensures that retraining stays in $\mathcal{W}$.

Throughout, $x_i$ denotes agent $i$'s block of a vector
$x\in\mathbb{R}^{d_1}\times\cdots\times\mathbb{R}^{d_k}$, with Euclidean norm
$\|\cdot\|$. The spectral radius of $M$ is $\rho(M)$, and for
$q\in\mathbb{R}_{++}^k$ we write
$\|x\|_q:=\max_i \|x_i\|/q_i$.

\section{The Reduction: From Retraining Dynamics to \texorpdfstring{$\Gamma$}{Gamma}}
\label{sec:reduction}

The retraining dynamics $w_{t+1}=B(w_t)$ are nonlinear and couple all $k$ agents through the shared endogenous distribution $D(w)$. Direct analysis of these dynamics is difficult without additional structure. This section introduces the reduction used throughout the paper: under the assumptions of Section~\ref{sec:model}, changes in the nonlinear map $B$ are bounded entrywise by a nonnegative matrix $\Gamma$. The spectral properties of $\Gamma$ can therefore be used to bound the dynamics of $B$.

The first step bounds the shift in any agent's best response induced by a shift in the deployed profile. The bound is linear and decomposes cleanly across agents.

\begin{theorem}[Best-response sensitivity]
\label{thm:br-sensitivity}
Under Assumptions~\ref{ass:sc}--\ref{ass:gs} and~\ref{ass:closed}, for every agent $i$ and every pair of deployed profiles $w, w'$,
\[
\|B_i(w) - B_i(w')\|
\;\le\;
\sum_{j=1}^k \frac{a_{ij}}{\mu_i}\,\|w_j - w_j'\|.
\]
\end{theorem}

\begin{proof}
Fix an agent $i$ and let $u_i := B_i(w)$ and $u_i' := B_i(w')$. First-order optimality gives $\nabla_{u_i} F_i(u_i; w) = 0$ and $\nabla_{u_i} F_i(u_i'; w') = 0$. Strong convexity of $F_i(\cdot; w)$ (Assumption~\ref{ass:sc}) yields $\mu_i\|u_i - u_i'\| \le \|\nabla_{u_i} F_i(u_i; w) - \nabla_{u_i} F_i(u_i'; w)\|$. Substituting the first optimality condition and then the second gives $\mu_i\|u_i - u_i'\| \le \|\nabla_{u_i} F_i(u_i'; w)\| = \|\nabla_{u_i} F_i(u_i'; w) - \nabla_{u_i} F_i(u_i'; w')\|$. The right-hand side is a gradient difference at a fixed action $u_i'$ across two deployed profiles, which Assumption~\ref{ass:gs} bounds by $\sum_j a_{ij}\,\|w_j - w_j'\|$. Dividing by $\mu_i > 0$ yields the claim.
\end{proof}

The proof uses each structural assumption once: gradient sensitivity bounds the performative change in agent $i$'s gradient, and strong convexity converts it into a best-response change. Their interaction appears only through $a_{ij}/\mu_i$, the ratio of agent $j$'s disturbance to agent $i$'s curvature, which measures how much $j$ can shift $i$'s best response.

The weights in Theorem~\ref{thm:br-sensitivity} define the central object of the paper.

\begin{definition}[Reflexivity matrix]
\label{def:gamma}
The \emph{reflexivity matrix} is the $k \times k$ nonnegative matrix $\Gamma$ with entries $\Gamma_{ij} := a_{ij}/\mu_i$, for $i, j = 1, \dots, k$.
\end{definition}

In terms of $\Gamma$, Theorem~\ref{thm:br-sensitivity} states that for every $i$ and every $w, w'$,
\begin{equation}
\label{eq:br-Gamma}
\|B_i(w) - B_i(w')\|
\;\le\;
\sum_{j=1}^k \Gamma_{ij}\,\|w_j - w_j'\|.
\end{equation}
The vector of best-response shifts is therefore dominated entrywise by the product of $\Gamma$ with the vector of action shifts.

The matrix $\Gamma$ is the central reduction object. It is built from primitive per-agent quantities rather than imposed as an aggregate Lipschitz constant, and its nonnegativity makes Perron--Frobenius theory available for the weighted-norm contraction argument below. The rank-one construction in Section~\ref{sec:lower-bound} saturates~\eqref{eq:br-Gamma} along the Perron direction, so the same spectral radius governs both the upper and lower sides of the threshold.

For a direct comparison with \citet{narang2023multiplayer}, consider scalar agents with quadratic losses $\ell_i(u_i;z_i)=\tfrac12(u_i-z_i)^2$ and deterministic observations $z=Cw$, where $C\ge0$. Here $\mu_i=1$, the data-gradient Lipschitz constants are $\beta_i=1$, and the full-profile Wasserstein sensitivities are $\gamma_i=\|C_{i:}\|_2$. Their scalar aggregate condition becomes $\sqrt{\sum_i\gamma_i^2}=\|C\|_F<1$, whereas ours is $\rho(C)<1$, since $\Gamma=C$. Because $\rho(C)\le\|C\|_F$, our certificate is never more restrictive in this common subclass and can be strictly stronger. For example, if $C=cP$ with $P$ a cyclic permutation matrix, then $\rho(C)=c$ but $\|C\|_F=c\sqrt{k}$. Thus, for $k^{-1/2}<c<1$, the matrix condition certifies convergence while the scalar aggregate condition does not.

\section{Subcritical Regime: Contraction and Bounded Excess Loss}
\label{sec:subcritical}

This section proves the positive side of the threshold. Under the assumptions of Section~\ref{sec:model}, if $\rho(\Gamma)<1$, then the retraining map is a contraction in a Perron--Frobenius weighted norm. This yields a unique performatively stable equilibrium, geometric convergence from any initialization, and a cumulative excess-loss bound that does not grow with the time horizon. 

\paragraph{Contraction via Perron--Frobenius.}
For two retraining trajectories $w_t$ and $w'_t$, let $d_{i,t}:=\|w_{i,t}-w'_{i,t}\|$. Inequality~\eqref{eq:br-Gamma} gives $d_{t+1}\le\Gamma d_t$ entrywise, hence $d_t\le\Gamma^t d_0$. Since $\rho(\Gamma)<1$, the matrix powers vanish, directly explaining why deviations decay. Perron--Frobenius theory supplies positive weights $q$ with $\Gamma q\le\alpha q$ for some $\alpha<1$, converting this matrix comparison into a one-step contraction in $\|\cdot\|_q$ and allowing Banach's theorem to establish existence and uniqueness of the equilibrium.

\begin{theorem}[Weighted contraction below threshold]
\label{thm:contraction}
Suppose Assumptions~\ref{ass:sc}--\ref{ass:gs} and~\ref{ass:closed} hold, and $\rho(\Gamma) < 1$. Then there exist a positive vector $q \in \mathbb{R}^k_{++}$ and a constant $\alpha \in (0, 1)$ such that
\[
\|B(w) - B(w')\|_q
\;\le\;
\alpha \, \|w - w'\|_q
\qquad \text{for all } w, w' \in \mathcal{W},
\]
where the weighted norm is defined by
\[
\|x\|_q \;:=\; \max_{1 \le i \le k} \frac{\|x_i\|}{q_i}.
\]
\end{theorem}

\begin{proof}
Since $\Gamma$ is nonnegative and $\rho(\Gamma) < 1$, Perron--Frobenius theory guarantees the existence of a vector $q \in \mathbb{R}^k_{++}$ and a constant $\alpha \in (0, 1)$ satisfying
\begin{equation}
\label{eq:pf}
\sum_{j=1}^k \Gamma_{ij} \, q_j \;\le\; \alpha \, q_i
\qquad \text{for every } i.
\end{equation}
If $\Gamma$ is reducible, apply the argument to the perturbed matrix $\Gamma_\eta := \Gamma + \eta \mathbf{1}\mathbf{1}^\top$ for $\eta > 0$ chosen small enough that $\rho(\Gamma_\eta) < 1$; the strictly positive Perron eigenvector $q$ of $\Gamma_\eta$ then satisfies $\Gamma q \le \Gamma_\eta q = \rho(\Gamma_\eta) q$, and we set $\alpha := \rho(\Gamma_\eta) \in (\rho(\Gamma), 1)$.

Fix $w, w' \in \mathcal{W}$ and an agent $i$. By the definition of the weighted norm, $\|w_j - w_j'\| \le q_j \, \|w - w'\|_q$ for every $j$. Substituting this bound into the best-response sensitivity inequality~\eqref{eq:br-Gamma} and then applying~\eqref{eq:pf} yields
\[
\|B_i(w) - B_i(w')\|
\;\le\;
\sum_{j=1}^k \Gamma_{ij} \, q_j \, \|w - w'\|_q
\;\le\;
\alpha \, q_i \, \|w - w'\|_q.
\]
Dividing by $q_i > 0$ and taking the maximum over $i$ gives the claim.
\end{proof}

Equivalently, for any $\alpha\in(\rho(\Gamma),1)$, one can choose positive weights $q$ such that $\Gamma q\le \alpha q$. The contraction factor can therefore be made arbitrarily close to $\rho(\Gamma)$. The proof uses only two facts: $\Gamma$ is nonnegative and $\rho(\Gamma)<1$.

The contraction theorem gives the standard fixed-point consequences. Since $B$ is a contraction on the complete space $(\mathcal{W},\|\cdot\|_q)$, Banach's fixed-point theorem applies.

\begin{corollary}[Existence, uniqueness, and geometric convergence]
\label{cor:convergence}
Under the conditions of Theorem~\ref{thm:contraction}:
\begin{enumerate}
\item The retraining map $B$ admits a unique fixed point $w^\star \in \mathcal{W}$, which is the unique performatively stable equilibrium of the system.
\item For every initialization $w_0 \in \mathcal{W}$, the iterates $w_{t+1} = B(w_t)$ satisfy
\[
\|w_t - w^\star\|_q
\;\le\;
\alpha^t \, \|w_0 - w^\star\|_q
\qquad \text{for every } t \ge 0.
\]
\end{enumerate}
\end{corollary}

\begin{proof}
Assumption~\ref{ass:closed} ensures that $\mathcal{W}$ is complete under $\|\cdot\|_q$ and that $B$ maps $\mathcal{W}$ into itself. Theorem~\ref{thm:contraction} shows that $B$ is an $\alpha$-contraction with $\alpha\in(0,1)$. Banach's fixed-point theorem yields the unique fixed point $w^\star$. Iterating the contraction inequality with $w=w_t$ and $w'=w^\star$ gives the geometric convergence bound.
\end{proof}

Corollary~\ref{cor:convergence} gives the basic learnability result in the subcritical regime. The system has a unique performatively stable equilibrium, and repeated retraining reaches it at a geometric rate from any initialization. The condition is stated through $\Gamma$, which is constructed from primitive per-agent quantities.

If retraining instead satisfies $\|w_{t+1}-B(w_t)\|_q\le\eta_t$, then
\[
\|w_{t+1}-w^\star\|_q\le\alpha\|w_t-w^\star\|_q+\eta_t.
\]
Appendix~\ref{app:approximate-retraining} gives the resulting error bound and its interpretation for empirical retraining with controlled gradient error.

Contraction of the state dynamics alone does not give a performance bound. To control cumulative excess loss, we add smoothness in each agent's own action.

\begin{assumption}[Smoothness in own action]
\label{ass:smooth}
For each agent $i$ and every fixed deployed profile $w$, the map $u_i \mapsto F_i(u_i; w)$ is $L_i$-smooth.
\end{assumption}

The natural benchmark in a time-varying environment is the current best response. At round $t$, agent $i$ deploys $w_{i,t}$, while the action minimizing its risk under the current profile $w_t$ is $B_i(w_t)$. The gap between these two actions is the instantaneous excess loss, and summing over agents and rounds gives

\begin{equation}
\label{eq:regret-def}
\widetilde{\mathrm{Reg}}_T
\;:=\;
\sum_{t=0}^{T-1} \sum_{i=1}^k
\Big( F_i(w_{i,t}; w_t) - F_i(B_i(w_t); w_t) \Big).
\end{equation}

This benchmark compares each deployment to the best response in the environment it actually faces. It also avoids assumptions on how objective values vary with the deployed profile. Bounds against the fixed equilibrium comparator require such assumptions and are therefore treated separately. 

\begin{corollary}[Bounded cumulative dynamic excess loss]
\label{cor:bounded-regret}
Under Assumptions~\ref{ass:sc}--\ref{ass:gs},~\ref{ass:closed}, and~\ref{ass:smooth}, and assuming $\rho(\Gamma) < 1$, there exists a constant $C > 0$ depending only on $\Gamma$, $q$, and $\{L_i\}$ such that
\[
\widetilde{\mathrm{Reg}}_T
\;\le\;
\frac{C}{1 - \alpha^2} \, \|w_0 - w^\star\|_q^2
\qquad \text{for every } T \ge 1.
\]
In particular, the cumulative excess loss is bounded uniformly in the horizon $T$.
\end{corollary}

\begin{proof}
Fix $t \ge 1$ and $i$. The best response $B_i(w_t)$ minimizes $u_i \mapsto F_i(u_i; w_t)$, which is $L_i$-smooth by Assumption~\ref{ass:smooth}. Standard smoothness gives
\begin{equation}
\label{eq:smooth-quad}
F_i(w_{i,t}; w_t) - F_i(B_i(w_t); w_t)
\;\le\;
\frac{L_i}{2} \, \|w_{i,t} - B_i(w_t)\|^2.
\end{equation}
For $t \ge 1$, the dynamics~\eqref{eq:dynamics} give $w_{i,t} = B_i(w_{t-1})$, so
\[
\|w_{i,t} - B_i(w_t)\|
\;=\;
\|B_i(w_{t-1}) - B_i(w_t)\|
\;\le\;
\sum_{j=1}^k \Gamma_{ij} \, \|w_{j,t-1} - w_{j,t}\|,
\]
where the inequality is~\eqref{eq:br-Gamma}. Using $\|w_{j,t}-w_{j,t-1}\|\le q_j\|w_t-w_{t-1}\|_q$ gives
\[
\|w_{i,t}-B_i(w_t)\|
\;\le\;
\Big(\sum_{j=1}^k \Gamma_{ij}q_j\Big)\|w_t-w_{t-1}\|_q.
\] 
Substituting into~\eqref{eq:smooth-quad} and summing over $i$ gives
\[
\sum_{i=1}^k
\Big( F_i(w_{i,t}; w_t) - F_i(B_i(w_t); w_t) \Big)
\;\le\;
C' \, \|w_t - w_{t-1}\|_q^2,
\]
for a constant $C'$ absorbing the $L_i$ and the row-sum factors $\big(\sum_j \Gamma_{ij} q_j\big)^2$. By Corollary~\ref{cor:convergence} and the triangle inequality,
\[
\|w_t - w_{t-1}\|_q
\;\le\;
\|w_t - w^\star\|_q + \|w_{t-1} - w^\star\|_q
\;\le\;
(1 + \alpha) \, \alpha^{t-1} \, \|w_0 - w^\star\|_q,
\]
so the per-round excess loss decays as $\alpha^{2(t-1)}$. Summing the geometric series from $t = 1$ to $T - 1$ gives a contribution of order $\|w_0 - w^\star\|_q^2 / (1 - \alpha^2)$. For $t = 0$, $\|w_0 - B(w_0)\|_q \le \|w_0 - w^\star\|_q + \|B(w^\star) - B(w_0)\|_q \le (1 + \alpha)\|w_0 - w^\star\|_q$, which fits the same geometric series. Combining yields the claim.
\end{proof}

The bound depends on the initialization and the contraction factor but not on the time horizon, with the $1/(1-\alpha^2)$ factor arising from the geometric decay of the one-step gap together with smoothness. Below the threshold, cumulative excess loss is thus $O(1)$; Section~\ref{sec:lower-bound} constructs supercritical instances with $\Omega(T)$ cumulative excess loss over a critical horizon.

The geometric bound gives iteration complexity $O\!\left(\log(\|w_0-w^\star\|_q/\delta)/(-\log\alpha)\right)$ to accuracy $\delta$. This bound and the factor $1/(1-\alpha^2)$ deteriorate as $\alpha\to1^-$. Figure~\ref{fig:threshold} illustrates the slowdown in a rescaled linear family as $\rho(\Gamma)$ approaches one.

\section{Supercritical Regime}
\label{sec:lower-bound}

Section~\ref{sec:subcritical} established that $\rho(\Gamma)<1$ is sufficient for geometric convergence and bounded cumulative excess loss. Whether the supercritical regime produces an intrinsic learning obstruction is delicate: failure of contraction does not by itself imply information-theoretic impossibility, since with enough observations a learner may recover the correct equilibrium even when repeated retraining diverges. We show that the threshold has a sharp negative side in the worst case. For every $\varepsilon>0$, we construct a supercritical family with $\rho(\Gamma)=1+\varepsilon$ in which any decentralized local-information learner---one that at each round uses only its own history of observations and actions---incurs linear cumulative excess loss for $T=O(1/\varepsilon)$ rounds.

\begin{figure}[t]
  \centering
  \begin{subfigure}[t]{0.48\linewidth}
    \centering
    \includegraphics[width=\linewidth]{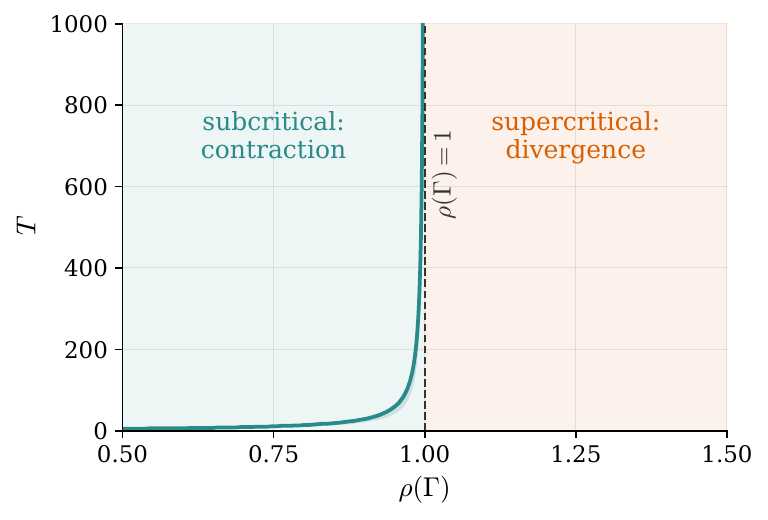}
    \caption{Number of retraining steps needed to satisfy $\|w_t-w^\star\|_q \le \delta$. Convergence slows sharply as $\rho(\Gamma)$ approaches $1$ from below. Only subcritical values are plotted, with a computational cap of $10{,}000$ iterations and a displayed vertical range up to $1{,}000$.}
    \label{fig:phase}
  \end{subfigure}
  \hfill
  \begin{subfigure}[t]{0.48\linewidth}
    \centering
    \includegraphics[width=\linewidth]{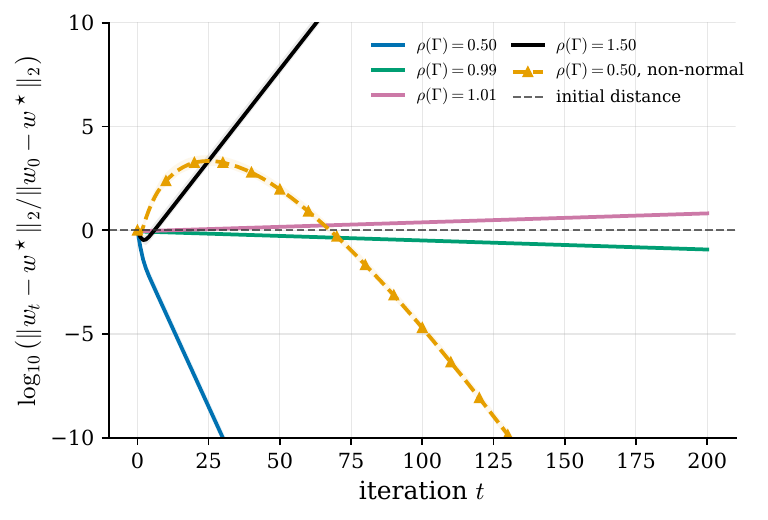}
    \caption{Normalized distance to equilibrium (log scale) under repeated retraining. Subcritical instances contract geometrically, while the plotted supercritical instances move away from equilibrium. The non-normal subcritical curve shows transient growth before contraction.}
    \label{fig:trajectories}
  \end{subfigure}
    \caption{\textbf{The reflexivity threshold.} 
    The simulations show that convergence slows as $\rho(\Gamma)$ approaches $1$ from below. 
    Below the threshold, repeated retraining moves geometrically toward equilibrium; 
    above the threshold, the plotted trajectories move away from equilibrium. 
    Setup details are in Appendix~\ref{app:simulations}.}
  \label{fig:threshold}
\end{figure}

All agents know the model parameters $(\Gamma_\varepsilon, \Delta, \sigma_\varepsilon^2, v)$; only the hidden parameter $\nu$ is unknown.

\begin{definition}[Decentralized local-information learner]
\label{def:decentralized}
A \emph{decentralized local-information learner} is a policy profile $A = (A_1, \dots, A_k)$ in which agent $i$'s action $w_{i,t}$ at round $t$ is measurable with respect to the private history $\mathcal{H}_{i,t}^{\mathrm{pre}} := \sigma\big( w_{i, 0:t-1}, \; Y_{i, 0:t-1}, \; U_{i, 0:t} \big)$, where $Y_{i,s}$ is agent $i$'s private observation at round $s$ and $U_{i,s}$ is agent $i$'s internal randomization.
\end{definition}

No agent observes the full system state: each sees only its own signal, without access to other agents' actions or observations. The lower bound applies to any learner satisfying this constraint, regardless of the algorithm used within each agent.

\paragraph{Hard family.}
The construction places all endogenous amplification along a single direction, which simplifies the analysis and makes the separation between the two environments explicit.

Fix $\varepsilon \in (0, 1]$ and a unit vector $v \in \mathbb{R}^k$ with nonnegative entries. Define the reflexivity matrix
\begin{equation}
\label{eq:rank-one}
\Gamma_\varepsilon \;:=\; (1 + \varepsilon) \, v v^\top.
\end{equation}
The matrix $\Gamma_\varepsilon$ is nonnegative, symmetric, and rank one, with single nonzero eigenvalue $1+\varepsilon$, so $\rho(\Gamma_\varepsilon)=1+\varepsilon$. The system is therefore supercritical by construction, with $\varepsilon$ measuring the distance above the threshold. Each agent has $d_i = 1$ and $\mathcal{W}_i = \mathbb{R}$, and the quadratic loss gives $\mu_i = 1$, so the abstract reflexivity matrix of Section~\ref{sec:reduction} coincides with $\Gamma_\varepsilon$ in this construction.

The environment contains a hidden binary parameter $\nu \in \{-1, +1\}$ that determines the true state of the world. Fix an amplitude $\Delta > 0$ and let the noise variance be $\sigma_\varepsilon^2 := 8 \Delta^2 / \varepsilon$. The observation model is
\begin{equation}
\label{eq:obs-model}
Y_t \;=\; \nu \Delta v + \Gamma_\varepsilon w_t + \xi_t,
\qquad
\xi_t \sim \mathcal{N}(0, \sigma_\varepsilon^2 I_k),
\end{equation}
where each agent $i$ observes only the coordinate $Y_{i,t}$.

Each agent uses quadratic loss $\ell_i(a; y) := \tfrac{1}{2}(a - y)^2$, which gives the performative risk
\begin{equation}
\label{eq:perf-risk}
F_i(a; w, \nu)
\;=\;
\frac{1}{2} \Big( a - \nu \Delta v_i - (\Gamma_\varepsilon w)_i \Big)^2
+ \frac{\sigma_\varepsilon^2}{2}.
\end{equation}
The additive constant $\sigma_\varepsilon^2/2$ cancels in differences and is dropped below.
Minimizing over $a$ gives the best-response map $B(w; \nu) = \nu \Delta v + \Gamma_\varepsilon w$. The fixed-point equation $w^{\star, \nu} = B(w^{\star, \nu}; \nu)$ has the explicit solution
\begin{equation}
\label{eq:equilibria}
w^{\star, \nu}
\;=\;
(I - \Gamma_\varepsilon)^{-1} \nu \Delta v
\;=\;
-\frac{\nu \Delta}{\varepsilon} \, v.
\end{equation}
The last equality uses the identity $(I - \Gamma_\varepsilon) v = -\varepsilon v$, which follows from the definition~\eqref{eq:rank-one}.

The two equilibria $w^{\star, +}$ and $w^{\star, -}$ are therefore located at distance $2 \Delta / \varepsilon$ along $v$. As $\varepsilon$ decreases, these equilibria move apart, and the geometric penalty for choosing the wrong sign grows. This divergence is the mechanism that converts slow information acquisition into large cumulative excess loss.

The instantaneous excess performative loss at round $t$ is
\begin{equation}
\label{eq:inst-regret}
r_t^\nu
\;:=\;
\sum_{i=1}^k
\Big( F_i(w_{i,t}; w_t, \nu) - F_i(w_i^{\star, \nu}; w^{\star, \nu}, \nu) \Big),
\end{equation}
and the cumulative excess loss over $T$ rounds is $\mathrm{Reg}_T^\nu(A) := \sum_{t=0}^{T-1} r_t^\nu$.

The subcritical result uses a dynamic comparator (the current best response $B_i(w_t)$), whereas~\eqref{eq:inst-regret} uses the fixed equilibrium comparator $w^{\star,\nu}$: the dynamic comparator measures one-step adaptation to the current caused environment, while the equilibrium comparator measures whether the learner identifies the correct stable equilibrium. In the supercritical construction these goals decouple, so the equilibrium comparator is the appropriate benchmark for the two-environment testing argument. For the quadratic family in~\eqref{eq:perf-risk} the two comparators in fact yield identical instantaneous excess loss, so the $\Omega(\Delta^2 T)$ bound holds under either benchmark.


\begin{theorem}[Linear lower bound in a supercritical family]
\label{thm:lower-bound}
There exist absolute constants $c_0, c_1 > 0$ such that the following holds. For every $\varepsilon \in (0, 1]$, in the supercritical family defined by~\eqref{eq:rank-one}--\eqref{eq:obs-model} with $\sigma_\varepsilon^2 = 8\Delta^2/\varepsilon$, every decentralized local-information learner $A$ satisfies
\[
\frac{1}{2} \, \mathbb{E}_+\big[ \mathrm{Reg}_T^+(A) \big]
+ \frac{1}{2} \, \mathbb{E}_-\big[ \mathrm{Reg}_T^-(A) \big]
\;\ge\;
c_1 \, \Delta^2 \, T
\qquad
\text{for every } T \le c_0 / \varepsilon.
\]
In particular,
\[
\sup_{\nu \in \{-1, +1\}} \mathbb{E}_\nu\big[ \mathrm{Reg}_T^\nu(A) \big]
\;\ge\;
c_1 \, \Delta^2 \, T
\qquad
\text{for every } T \le c_0 / \varepsilon.
\]
\end{theorem}

The proof proceeds in three steps. Step~1 shows that the two environments $\nu = +1$ and $\nu = -1$ remain statistically indistinguishable over the critical horizon. Step~2 shows that misidentifying the environment forces the learner into a region of large excess loss. Step~3 combines these two facts.

\begin{proof}[Proof sketch; full proof in Appendix~\ref{app:lb}.]
The argument has two parts.

\emph{Indistinguishability.} We prove the KL bound under a stronger centralized observation model that reveals the full profile $w_t$ and vector observation $Y_t$; the bound therefore also applies to decentralized learners. Conditional on the realized pre-observation history $\mathcal{F}_t$, the action $w_t$ is fixed by the learner's policy. Under the two hypotheses, the next observation is Gaussian with common covariance $\sigma_\varepsilon^2 I_k$ and conditional means differing by $2\Delta v$. The per-round Kullback--Leibler divergence is therefore $\|2\Delta v\|^2/(2\sigma_\varepsilon^2) = \varepsilon/4$, using $\|v\|=1$ and $\sigma_\varepsilon^2=8\Delta^2/\varepsilon$. By the chain rule for KL divergence and Pinsker's inequality, the total variation distance remains bounded away from one for all $t\le c_0/\varepsilon$, so the optimal equal-prior testing error is bounded below by an absolute constant $p_0>0$.

\emph{Geometric cost of misidentification.} The two equilibria~\eqref{eq:equilibria} lie at $\pm(\Delta/\varepsilon) v$, so on the event that the midpoint decoder picks the wrong sign, $|\langle w_t - w^{\star,\nu}, v \rangle| \ge \Delta/\varepsilon$. The closed form $r_t^\nu = \tfrac{1}{2}\|(I-\Gamma_\varepsilon)(w_t - w^{\star,\nu})\|^2$ together with $(I-\Gamma_\varepsilon)v = -\varepsilon v$ then gives $r_t^\nu \ge \Delta^2/2$ on that event. The factor $\varepsilon$ from the contraction of $(I-\Gamma_\varepsilon)$ along $v$ cancels exactly with the factor $1/\varepsilon$ from the equilibrium separation, which is why the obstruction survives the limit $\varepsilon \to 0$ and produces a horizon of order $1/\varepsilon$.

Combining the two parts: averaging over $\nu$ and using equal-prior testing error at least $p_0$ gives $\tfrac{1}{2}\mathbb{E}_+[r_t^+] + \tfrac{1}{2}\mathbb{E}_-[r_t^-] \ge p_0\Delta^2/2$ for every $t \le c_0/\varepsilon$, and summing yields the claim with $c_1 = p_0/2$.
\end{proof}

The positive theorem assumes exact population best responses, whereas the lower bound concerns noisy observations over a finite horizon, with $\sigma_\varepsilon^2=8\Delta^2/\varepsilon$. The proof bounds indistinguishability even for a centralized learner observing the same signals, so communication alone does not remove this obstruction in the constructed family.

\section{Discussion}
\label{sec:discussion}

\paragraph{Reflexivity matrix and interpretation.}
The framework replaces a scalar Lipschitz constant with the reflexivity matrix $\Gamma$, whose entries $a_{ij}/\mu_i$ retain the per-pair structure of performative feedback. This matters because two systems with the same scalar bound can have different stability behavior if their feedback is routed through different agent networks. The Perron weights $q$ in Theorem~\ref{thm:contraction} encode this heterogeneity by rescaling agents according to their role in the collective retraining dynamics. When $k=1$, $\Gamma$ reduces to the scalar ratio $a_{11}/\mu_1$, recovering the single-agent contraction threshold of \citet{perdomo2020performative}. More generally, $\rho(\Gamma)=1$ separates dissipative feedback, where repeated retraining contracts, from supercritical feedback, where uniform stability guarantees fail. We use \emph{reflexivity} for this deployment-to-data channel, following the financial-market terminology of \citet{soros1987alchemy}; we do not formalize the broader economic concept. 

\paragraph{Beyond the spectral radius.}
The threshold $\rho(\Gamma)=1$ governs asymptotic behavior, but finite-horizon dynamics can depend on finer matrix properties and initial conditions. One important case is non-normality: $\|\Gamma^t\|$ may grow transiently even if $\rho(\Gamma)<1$, as illustrated by the non-normal subcritical trajectory in Panel~\ref{fig:trajectories}. Thus the spectral radius identifies the asymptotic threshold, while quantities such as the numerical range or pseudospectra may better capture finite-horizon amplification. We leave this refinement to future work.

\paragraph{Empirical measurement.}
The framework identifies a stability quantity, $\rho(\Gamma)$, that is measurable in principle but difficult to identify empirically. Each entry $a_{ij}$ is counterfactual: it describes how agent $i$'s gradient would respond to alternative deployments by agent $j$, while $\mu_i$ is also not directly observed. Estimation is therefore a causal identification problem, since common system shocks may affect both agent $j$'s deployment and agent $i$'s gradient. Testing whether estimated values of $\rho(\Gamma)$ predict observed instability is a concrete direction for future work. The primitive bounds are useful when $B$ is defined implicitly by a nonlinear optimizer: curvature and pairwise gradient sensitivity can be bounded without solving or differentiating the full best-response map. The certificate can be conservative and depends on the chosen block geometry; failure of $\rho(\Gamma)<1$ alone does not diagnose instability. As a design lever, adding a fixed-reference quadratic penalty $\tfrac{\lambda_i}{2}\|u_i-u_i^{\mathrm{ref}}\|^2$ replaces $\Gamma_{ij}$ by $a_{ij}/(\mu_i+\lambda_i)$, potentially restoring the certificate, while changing the objective and generally the equilibrium.

\paragraph{Limitations and extensions.}
Several modeling choices delimit our results. The main guarantees assume strong convexity in each agent's own action and exact best-response retraining; Appendix~\ref{app:approximate-retraining} controls approximate retraining under explicit error bounds. Sharper finite-sample rates, weaker curvature assumptions, and gradient updates remain directions for future work. We also treat the reflexivity matrix as fixed across rounds. If the interaction structure evolves, a natural extension is to study time-varying matrices $\Gamma_t$, where tools such as joint spectral radii may be needed to characterize stability. The lower bound establishes that no universal positive guarantee can extend above the threshold, since at least one supercritical family forces linear cumulative excess loss; it does not assert that every supercritical $\Gamma$ is unlearnable. Characterizing which supercritical systems remain learnable is open for future work. Coordination among agents may shift the threshold, while damped retraining is a natural design lever for stabilizing otherwise unstable systems.

\begin{ack}
This work received no funding. The authors declare no competing interests.
\end{ack}

\bibliographystyle{plainnat}
\bibliography{references}

\appendix

\section{Full Proof of Theorem~\ref{thm:lower-bound}}
\label{app:lb}

This appendix gives the full proof of Theorem~\ref{thm:lower-bound}, expanding the
sketch in Section~\ref{sec:lower-bound}. We retain the same notation: $P_\nu$
denotes the law of the full system trajectory under learner $A$ in environment
$\nu \in \{-1, +1\}$, and $\mathcal{F}_t$ denotes the history up to and including
the action $w_t$ but excluding the observation $Y_t$.

\begin{proof}[Proof of Theorem~\ref{thm:lower-bound}]
\emph{Step 1: Bounding the Kullback--Leibler divergence.}
The KL bound is proved under a stronger centralized observation model in which the learner sees the full profile $w_t$ and vector observation $Y_t$. Since decentralized learners observe only part of this information, the same indistinguishability bound applies to the learners in Theorem~\ref{thm:lower-bound}. Conditional on the realized pre-observation history $\mathcal{F}_t$, the action $w_t$ is fixed by the learner's policy. The observation $Y_t$ is then drawn from a Gaussian whose mean differs between the two environments only through the term $\nu \Delta v$. Specifically,
\[
P_\nu(Y_t \mid \mathcal{F}_t)
\;=\;
\mathcal{N}\big( \nu \Delta v + \Gamma_\varepsilon w_t, \, \sigma_\varepsilon^2 I_k \big).
\]
The means differ by $2 \Delta v$, and the covariance matrices are identical. The Kullback--Leibler divergence between two Gaussians with common covariance $\sigma^2 I$ and mean difference $\mu$ equals $\|\mu\|^2 / (2 \sigma^2)$. Applying this formula and using $\sigma_\varepsilon^2 = 8\Delta^2/\varepsilon$,
\[
\mathrm{KL}\big( P_+(Y_t \mid \mathcal{F}_t) \,\|\, P_-(Y_t \mid \mathcal{F}_t) \big)
\;=\;
\frac{\|2 \Delta v\|^2}{2 \sigma_\varepsilon^2}
\;=\;
\frac{2 \Delta^2}{\sigma_\varepsilon^2}
\;=\;
\frac{\varepsilon}{4}.
\]
By the chain rule for Kullback--Leibler divergence, the total divergence after $t$ rounds is
\begin{equation}
\label{eq:kl-chain}
\mathrm{KL}\big( P_+^t \,\|\, P_-^t \big)
\;=\;
\frac{t \varepsilon}{4}.
\end{equation}
Choose $c_0 := 1/2$. For every $t \le c_0 / \varepsilon$, Pinsker's inequality gives
\[
\mathrm{TV}(P_+^t, P_-^t)
\;\le\;
\sqrt{ \frac{1}{2} \mathrm{KL}(P_+^t \,\|\, P_-^t) }
\;\le\;
\sqrt{ \frac{t \varepsilon}{8} }
\;\le\;
\sqrt{1/16}
\;=\;
\frac{1}{4}.
\]
The minimax testing error for deciding $\nu$ from the observed history is therefore bounded below by an absolute constant $p_0 > 0$, independent of $\varepsilon$ and $t$.

\emph{Step 2: Geometric cost of misidentification.}
Define the midpoint decoder
\[
\widehat{\nu}_t
\;:=\;
\argmin_{s \in \{-1, +1\}} \|w_t - w^{\star, s}\|.
\]
On the event $\{\widehat{\nu}_t \neq \nu\}$, the action $w_t$ lies at least as close to the incorrect equilibrium as to the correct one. Using~\eqref{eq:equilibria}, the two equilibria lie on opposite sides of the origin along $v$ at distance $\Delta / \varepsilon$ each. The midpoint between them is the origin, so $\widehat{\nu}_t \neq \nu$ implies $\nu\langle w_t,v\rangle\ge0$. Consequently,
\begin{equation}
\label{eq:wrong-side}
|\langle w_t - w^{\star, \nu}, v \rangle|
\;\ge\;
\frac{\Delta}{\varepsilon}
\qquad
\text{on } \{\widehat{\nu}_t \neq \nu\}.
\end{equation}
For the quadratic family, the equilibrium identity $w^{\star,\nu} = \nu \Delta v + \Gamma_\varepsilon w^{\star,\nu}$ implies $F_i(w_i^{\star,\nu}; w^{\star,\nu}, \nu) = \sigma_\varepsilon^2 / 2$. Substituting into~\eqref{eq:inst-regret} and summing over $i$ gives
\begin{equation}
\label{eq:regret-form}
r_t^\nu
\;=\;
\frac{1}{2} \, \|w_t - \nu \Delta v - \Gamma_\varepsilon w_t\|^2
\;=\;
\frac{1}{2} \, \|(I - \Gamma_\varepsilon)(w_t - w^{\star, \nu})\|^2,
\end{equation}
where the second equality uses $\nu \Delta v = (I - \Gamma_\varepsilon) w^{\star,\nu}$.

The identity $(I - \Gamma_\varepsilon) v = -\varepsilon v$ implies that $(I - \Gamma_\varepsilon)$ contracts the direction $v$ by the factor $\varepsilon$. Combining~\eqref{eq:wrong-side} with the lower bound
\[
\|(I - \Gamma_\varepsilon)(w_t - w^{\star, \nu})\|
\;\ge\;
| \langle (I - \Gamma_\varepsilon)(w_t - w^{\star, \nu}), v \rangle |
\;=\;
\varepsilon \, |\langle w_t - w^{\star, \nu}, v \rangle|,
\]
and substituting~\eqref{eq:wrong-side},
\[
\|(I - \Gamma_\varepsilon)(w_t - w^{\star, \nu})\|
\;\ge\;
\varepsilon \cdot \frac{\Delta}{\varepsilon}
\;=\;
\Delta
\qquad
\text{on } \{\widehat{\nu}_t \neq \nu\}.
\]
Substituting into~\eqref{eq:regret-form} yields
\begin{equation}
\label{eq:regret-on-event}
r_t^\nu \;\ge\; \frac{\Delta^2}{2}
\qquad
\text{on } \{\widehat{\nu}_t \neq \nu\}.
\end{equation}
The factor $\varepsilon$ from the contraction of $(I - \Gamma_\varepsilon)$ and the factor $1 / \varepsilon$ from the equilibrium separation cancel exactly. This cancellation is why the lower bound survives the limit $\varepsilon \to 0$ and produces a horizon of order $1 / \varepsilon$ rather than a faster-vanishing horizon.

\emph{Step 3: Combining the bounds.}
Taking expectation of~\eqref{eq:regret-on-event} under environment $\nu$ gives
\[
\mathbb{E}_\nu[r_t^\nu]
\;\ge\;
\frac{\Delta^2}{2} \, P_\nu(\widehat{\nu}_t \neq \nu).
\]
Averaging over $\nu \in \{-1, +1\}$,
\[
\frac{1}{2} \mathbb{E}_+[r_t^+] + \frac{1}{2} \mathbb{E}_-[r_t^-]
\;\ge\;
\frac{\Delta^2}{2} \cdot
\frac{1}{2} \Big( P_+(\widehat{\nu}_t \neq +1) + P_-(\widehat{\nu}_t \neq -1) \Big).
\]
The bracketed quantity is twice this decoder's equal-prior error, which is at least the optimal equal-prior testing error and hence at least $p_0$ by Step~1 for every $t \le c_0 / \varepsilon$. Therefore,
\[
\frac{1}{2} \mathbb{E}_+[r_t^+] + \frac{1}{2} \mathbb{E}_-[r_t^-]
\;\ge\;
\frac{p_0 \Delta^2}{2}
\qquad
\text{for every } t \le c_0 / \varepsilon.
\]
Summing over $t = 0, \dots, T - 1$ gives the first claim with $c_1 := p_0 / 2$. The supremum bound follows because the maximum over $\nu$ dominates the average.
\end{proof}

\section{Simulation Details for Figure~\ref{fig:threshold}}
\label{app:simulations}

All instances have $k=100$. The sparse matrices have zero diagonal; each off-diagonal position is independently selected with probability $0.15$ and assigned a uniform $[0,1]$ weight, before rescaling to the target spectral radius. The non-normal instance is constructed separately as described below. Code, seeds, and commands to reproduce the figures are provided in the supplementary simulation artifact.

\paragraph{Panel~\ref{fig:phase} (phase diagram).} With seed $0$, we generate $20$ sparse base matrices with spectral radius one. Each is paired with $5$ Gaussian initial states normalized to $\|w_0\|_q=1$ using its positive right Perron vector $q$. The dynamics are $w_{t+1}=\Gamma w_t$ with $w^\star=0$. For each $\rho$ on a grid of $200$ values in $[0.5,1.5]$, densely sampled near one, we rescale the base to $\Gamma=\rho\Gamma_{\mathrm{base}}$ and record the first iteration at which $\|w_t\|_q\le10^{-3}$, with computational cap $T_{\max}=10{,}000$. The curve reports the median and interquartile band across the $100$ trajectories. Only $\rho<1$ values are evaluated for hitting times and plotted; the vertical display is limited to $1{,}000$. The same Perron weights are used across rescalings of each base.

\paragraph{Panel~\ref{fig:trajectories} (trajectories).} With seed $20260439$, each sparse baseline at $\rho\in\{0.50,0.99,1.01,1.50\}$ uses $100$ independent draws of a matrix and of vectors $b,w_0$ whose coordinates are Gaussian with standard deviation $1/\sqrt{k}$. We simulate $w_{t+1}=b+\Gamma w_t$ for $200$ steps and compute $w^\star=(I-\Gamma)^{-1}b$. The non-normal instance uses a fixed dense upper-triangular $\Gamma=0.5I+sU$, where $U$ has ones strictly above the diagonal and $s$ is chosen so that $\|\Gamma\|_2/\rho(\Gamma)=10$. It uses $b=0$ and $100$ Gaussian initial states normalized to Euclidean unit norm. Each curve reports the mean of $\log_{10}(\|w_t-w^\star\|_2/\|w_0-w^\star\|_2+10^{-12})$, with pointwise interquartile bands. The horizontal zero line marks the initial normalized distance.

Since the experiments are synthetic simulations, they can be reproduced with standard CPU resources.

\clearpage
\section{Approximate and Empirical Retraining}
\label{app:approximate-retraining}

\begin{proposition}[Robustness to retraining errors]
\label{prop:approximate-retraining}
Under the conditions of Theorem~\ref{thm:contraction}, let $q,\alpha$ be its weights and contraction factor, and let $w^\star$ be the unique equilibrium. Suppose the deployed iterates satisfy $\|w_{t+1}-B(w_t)\|_q\le\eta_t$ with $\eta_t\ge0$. Then, for every $t\ge0$,
\[
\|w_t-w^\star\|_q\le\alpha^t\|w_0-w^\star\|_q+
\sum_{s=0}^{t-1}\alpha^{t-1-s}\eta_s.
\]
If $\sup_t\eta_t\le\bar\eta$, then $\limsup_{t\to\infty}\|w_t-w^\star\|_q\le\bar\eta/(1-\alpha)$; if $\eta_t\to0$, the iterates converge to $w^\star$.

In particular, suppose agent $i$ retrains to an exact minimizer $w_{i,t+1}$ of a differentiable empirical risk $\widehat F_{i,t}$ on $\mathbb{R}^{d_i}$. On any event where
\[
\sup_{u_i\in\mathbb{R}^{d_i}}
\|\nabla\widehat F_{i,t}(u_i)-\nabla_{u_i}F_i(u_i;w_t)\|\le g_{i,t}
\quad\text{for every }i,t,
\]
the same conclusions hold with $\eta_t=\max_i g_{i,t}/(\mu_iq_i)$.
\end{proposition}

\begin{proof}
The triangle inequality and Theorem~\ref{thm:contraction} give
\[
\|w_{t+1}-w^\star\|_q\le
\|w_{t+1}-B(w_t)\|_q+\|B(w_t)-B(w^\star)\|_q
\le\eta_t+\alpha\|w_t-w^\star\|_q.
\]
Iterating this inequality proves the bound. For bounded errors, the geometric sum is at most $\bar\eta/(1-\alpha)$. For vanishing errors, split the sum into a finite prefix, which vanishes as $t\to\infty$, and a tail bounded by $\sup_{s\ge S}\eta_s/(1-\alpha)$, then let $S\to\infty$.

For empirical retraining, first-order optimality gives $\nabla\widehat F_{i,t}(w_{i,t+1})=0$. Strong convexity of $F_i(\cdot;w_t)$ and $\nabla_{u_i}F_i(B_i(w_t);w_t)=0$ imply
\[
\mu_i\|w_{i,t+1}-B_i(w_t)\|
\le\|\nabla_{u_i}F_i(w_{i,t+1};w_t)\|
\le g_{i,t}.
\]
Dividing by $\mu_iq_i$ and maximizing over agents gives the claimed $\eta_t$.
\end{proof}

This is a conditional stability statement: obtaining probabilities or sample-size rates for the gradient-error event requires additional statistical assumptions on the losses and samples.

\newpage
\section*{NeurIPS Paper Checklist}

\begin{enumerate}

\item {\bf Claims}
    \item[] Question: Do the main claims made in the abstract and introduction accurately reflect the paper's contributions and scope?
    \item[] Answer: \answerYes{}
    \item[] Justification: The abstract and introduction state the main theoretical claims: the reflexivity matrix framework, the subcritical contraction result, and the supercritical lower bound. The scope is qualified through the stated assumptions, including strong convexity, exact best-response retraining, and decentralized local information.
    \item[] Guidelines:
    \begin{itemize}
        \item The answer \answerNA{} means that the abstract and introduction do not include the claims made in the paper.
        \item The abstract and/or introduction should clearly state the claims made, including the contributions made in the paper and important assumptions and limitations. A \answerNo{} or \answerNA{} answer to this question will not be perceived well by the reviewers. 
        \item The claims made should match theoretical and experimental results, and reflect how much the results can be expected to generalize to other settings. 
        \item It is fine to include aspirational goals as motivation as long as it is clear that these goals are not attained by the paper. 
    \end{itemize}

\item {\bf Limitations}
    \item[] Question: Does the paper discuss the limitations of the work performed by the authors?
    \item[] Answer: \answerYes{}
    \item[] Justification: The paper discusses limitations in Section 7.
    \item[] Guidelines:
    \begin{itemize}
        \item The answer \answerNA{} means that the paper has no limitation while the answer \answerNo{} means that the paper has limitations, but those are not discussed in the paper. 
        \item The authors are encouraged to create a separate ``Limitations'' section in their paper.
        \item The paper should point out any strong assumptions and how robust the results are to violations of these assumptions (e.g., independence assumptions, noiseless settings, model well-specification, asymptotic approximations only holding locally). The authors should reflect on how these assumptions might be violated in practice and what the implications would be.
        \item The authors should reflect on the scope of the claims made, e.g., if the approach was only tested on a few datasets or with a few runs. In general, empirical results often depend on implicit assumptions, which should be articulated.
        \item The authors should reflect on the factors that influence the performance of the approach. For example, a facial recognition algorithm may perform poorly when image resolution is low or images are taken in low lighting. Or a speech-to-text system might not be used reliably to provide closed captions for online lectures because it fails to handle technical jargon.
        \item The authors should discuss the computational efficiency of the proposed algorithms and how they scale with dataset size.
        \item If applicable, the authors should discuss possible limitations of their approach to address problems of privacy and fairness.
        \item While the authors might fear that complete honesty about limitations might be used by reviewers as grounds for rejection, a worse outcome might be that reviewers discover limitations that aren't acknowledged in the paper. The authors should use their best judgment and recognize that individual actions in favor of transparency play an important role in developing norms that preserve the integrity of the community. Reviewers will be specifically instructed to not penalize honesty concerning limitations.
    \end{itemize}

\item {\bf Theory assumptions and proofs}
    \item[] Question: For each theoretical result, does the paper provide the full set of assumptions and a complete (and correct) proof?
    \item[] Answer: \answerYes{}
    \item[] Justification: The theoretical results state their assumptions explicitly, including strong convexity, gradient sensitivity, unconstrained domains, smoothness, and the decentralized information model. Proofs are provided in the main text for the subcritical results, and the full proof of the lower bound is provided in Appendix A.
    \item[] Guidelines:
    \begin{itemize}
        \item The answer \answerNA{} means that the paper does not include theoretical results. 
        \item All the theorems, formulas, and proofs in the paper should be numbered and cross-referenced.
        \item All assumptions should be clearly stated or referenced in the statement of any theorems.
        \item The proofs can either appear in the main paper or the supplemental material, but if they appear in the supplemental material, the authors are encouraged to provide a short proof sketch to provide intuition. 
        \item Inversely, any informal proof provided in the core of the paper should be complemented by formal proofs provided in appendix or supplemental material.
        \item Theorems and Lemmas that the proof relies upon should be properly referenced. 
    \end{itemize}

    \item {\bf Experimental result reproducibility}
    \item[] Question: Does the paper fully disclose all the information needed to reproduce the main experimental results of the paper to the extent that it affects the main claims and/or conclusions of the paper (regardless of whether the code and data are provided or not)?
    \item[] Answer: \answerYes{}
    \item[] Justification: The simulation setup details are described in Appendix~\ref{app:simulations}. The simulation code, seeds, dependencies, and reproduction commands are provided in the supplementary material.
    \item[] Guidelines:
    \begin{itemize}
        \item The answer \answerNA{} means that the paper does not include experiments.
        \item If the paper includes experiments, a \answerNo{} answer to this question will not be perceived well by the reviewers: Making the paper reproducible is important, regardless of whether the code and data are provided or not.
        \item If the contribution is a dataset and\slash or model, the authors should describe the steps taken to make their results reproducible or verifiable. 
        \item Depending on the contribution, reproducibility can be accomplished in various ways. For example, if the contribution is a novel architecture, describing the architecture fully might suffice, or if the contribution is a specific model and empirical evaluation, it may be necessary to either make it possible for others to replicate the model with the same dataset, or provide access to the model. In general. releasing code and data is often one good way to accomplish this, but reproducibility can also be provided via detailed instructions for how to replicate the results, access to a hosted model (e.g., in the case of a large language model), releasing of a model checkpoint, or other means that are appropriate to the research performed.
        \item While NeurIPS does not require releasing code, the conference does require all submissions to provide some reasonable avenue for reproducibility, which may depend on the nature of the contribution. For example
        \begin{enumerate}
            \item If the contribution is primarily a new algorithm, the paper should make it clear how to reproduce that algorithm.
            \item If the contribution is primarily a new model architecture, the paper should describe the architecture clearly and fully.
            \item If the contribution is a new model (e.g., a large language model), then there should either be a way to access this model for reproducing the results or a way to reproduce the model (e.g., with an open-source dataset or instructions for how to construct the dataset).
            \item We recognize that reproducibility may be tricky in some cases, in which case authors are welcome to describe the particular way they provide for reproducibility. In the case of closed-source models, it may be that access to the model is limited in some way (e.g., to registered users), but it should be possible for other researchers to have some path to reproducing or verifying the results.
        \end{enumerate}
    \end{itemize}

\item {\bf Open access to data and code}
    \item[] Question: Does the paper provide open access to the data and code, with sufficient instructions to faithfully reproduce the main experimental results, as described in supplemental material?
    \item[] Answer: \answerYes{}
    \item[] Justification: The simulations use synthetic matrices and do not rely on external datasets. The supplementary material includes the simulation code, seeds, dependencies, and instructions to reproduce Figure~\ref{fig:threshold}.
    \item[] Guidelines:
    \begin{itemize}
        \item The answer \answerNA{} means that paper does not include experiments requiring code.
        \item Please see the NeurIPS code and data submission guidelines (\url{https://neurips.cc/public/guides/CodeSubmissionPolicy}) for more details.
        \item While we encourage the release of code and data, we understand that this might not be possible, so \answerNo{} is an acceptable answer. Papers cannot be rejected simply for not including code, unless this is central to the contribution (e.g., for a new open-source benchmark).
        \item The instructions should contain the exact command and environment needed to run to reproduce the results. See the NeurIPS code and data submission guidelines (\url{https://neurips.cc/public/guides/CodeSubmissionPolicy}) for more details.
        \item The authors should provide instructions on data access and preparation, including how to access the raw data, preprocessed data, intermediate data, and generated data, etc.
        \item The authors should provide scripts to reproduce all experimental results for the new proposed method and baselines. If only a subset of experiments are reproducible, they should state which ones are omitted from the script and why.
        \item At submission time, to preserve anonymity, the authors should release anonymized versions (if applicable).
        \item Providing as much information as possible in supplemental material (appended to the paper) is recommended, but including URLs to data and code is permitted.
    \end{itemize}

\item {\bf Experimental setting/details}
    \item[] Question: Does the paper specify all the training and test details (e.g., data splits, hyperparameters, how they were chosen, type of optimizer) necessary to understand the results?
    \item[] Answer: \answerYes{}
    \item[] Justification: Details needed to understand the results of the simulations are given in the main text. Appendix B specifies the full details for Figure 1. The code is also provided.
    \item[] Guidelines:
    \begin{itemize}
        \item The answer \answerNA{} means that the paper does not include experiments.
        \item The experimental setting should be presented in the core of the paper to a level of detail that is necessary to appreciate the results and make sense of them.
        \item The full details can be provided either with the code, in appendix, or as supplemental material.
    \end{itemize}

\item {\bf Experiment statistical significance}
    \item[] Question: Does the paper report error bars suitably and correctly defined or other appropriate information about the statistical significance of the experiments?
    \item[] Answer: \answerYes{}
    \item[] Justification: Panel~\ref{fig:phase} reports medians with interquartile bands across base matrices and initializations. Panel~\ref{fig:trajectories} reports mean log-distances with pointwise interquartile bands across the draws described in Appendix~\ref{app:simulations}.
    \item[] Guidelines:
    \begin{itemize}
        \item The answer \answerNA{} means that the paper does not include experiments.
        \item The authors should answer \answerYes{} if the results are accompanied by error bars, confidence intervals, or statistical significance tests, at least for the experiments that support the main claims of the paper.
        \item The factors of variability that the error bars are capturing should be clearly stated (for example, train/test split, initialization, random drawing of some parameter, or overall run with given experimental conditions).
        \item The method for calculating the error bars should be explained (closed form formula, call to a library function, bootstrap, etc.)
        \item The assumptions made should be given (e.g., Normally distributed errors).
        \item It should be clear whether the error bar is the standard deviation or the standard error of the mean.
        \item It is OK to report 1-sigma error bars, but one should state it. The authors should preferably report a 2-sigma error bar than state that they have a 96\% CI, if the hypothesis of Normality of errors is not verified.
        \item For asymmetric distributions, the authors should be careful not to show in tables or figures symmetric error bars that would yield results that are out of range (e.g., negative error rates).
        \item If error bars are reported in tables or plots, the authors should explain in the text how they were calculated and reference the corresponding figures or tables in the text.
    \end{itemize}

\item {\bf Experiments compute resources}
    \item[] Question: For each experiment, does the paper provide sufficient information on the computer resources (type of compute workers, memory, time of execution) needed to reproduce the experiments?
    \item[] Answer: \answerNo{}
    \item[] Justification: The paper does not currently report compute resources for the simulations. Since the experiments are synthetic simulations rather than large-scale training runs, they can be reproduced with standard CPU resources; this is stated in Appendix B.
    \item[] Guidelines:
    \begin{itemize}
        \item The answer \answerNA{} means that the paper does not include experiments.
        \item The paper should indicate the type of compute workers CPU or GPU, internal cluster, or cloud provider, including relevant memory and storage.
        \item The paper should provide the amount of compute required for each of the individual experimental runs as well as estimate the total compute. 
        \item The paper should disclose whether the full research project required more compute than the experiments reported in the paper (e.g., preliminary or failed experiments that didn't make it into the paper). 
    \end{itemize}
    
\item {\bf Code of ethics}
    \item[] Question: Does the research conducted in the paper conform, in every respect, with the NeurIPS Code of Ethics \url{https://neurips.cc/public/EthicsGuidelines}?
    \item[] Answer: \answerYes{}
    \item[] Justification: The work is theoretical and uses synthetic simulations only. It does not involve human subjects, private data, scraped data, deployed systems, or high-risk model release, and the authors have reviewed the NeurIPS Code of Ethics.
    \item[] Guidelines:
    \begin{itemize}
        \item The answer \answerNA{} means that the authors have not reviewed the NeurIPS Code of Ethics.
        \item If the authors answer \answerNo, they should explain the special circumstances that require a deviation from the Code of Ethics.
        \item The authors should make sure to preserve anonymity (e.g., if there is a special consideration due to laws or regulations in their jurisdiction).
    \end{itemize}

\item {\bf Broader impacts}
    \item[] Question: Does the paper discuss both potential positive societal impacts and negative societal impacts of the work performed?
    \item[] Answer: \answerNo{}
    \item[] Justification: The paper is foundational and is not tied to a deployed system. We do not identify a clear direct path to negative applications, and the paper does not include a dedicated discussion of positive and negative societal impacts.
    \item[] Guidelines:
    \begin{itemize}
        \item The answer \answerNA{} means that there is no societal impact of the work performed.
        \item If the authors answer \answerNA{} or \answerNo, they should explain why their work has no societal impact or why the paper does not address societal impact.
        \item Examples of negative societal impacts include potential malicious or unintended uses (e.g., disinformation, generating fake profiles, surveillance), fairness considerations (e.g., deployment of technologies that could make decisions that unfairly impact specific groups), privacy considerations, and security considerations.
        \item The conference expects that many papers will be foundational research and not tied to particular applications, let alone deployments. However, if there is a direct path to any negative applications, the authors should point it out. For example, it is legitimate to point out that an improvement in the quality of generative models could be used to generate Deepfakes for disinformation. On the other hand, it is not needed to point out that a generic algorithm for optimizing neural networks could enable people to train models that generate Deepfakes faster.
        \item The authors should consider possible harms that could arise when the technology is being used as intended and functioning correctly, harms that could arise when the technology is being used as intended but gives incorrect results, and harms following from (intentional or unintentional) misuse of the technology.
        \item If there are negative societal impacts, the authors could also discuss possible mitigation strategies (e.g., gated release of models, providing defenses in addition to attacks, mechanisms for monitoring misuse, mechanisms to monitor how a system learns from feedback over time, improving the efficiency and accessibility of ML).
    \end{itemize}
    
\item {\bf Safeguards}
    \item[] Question: Does the paper describe safeguards that have been put in place for responsible release of data or models that have a high risk for misuse (e.g., pre-trained language models, image generators, or scraped datasets)?
    \item[] Answer: \answerNA{}
    \item[] Justification: The paper does not release artifacts with substantial misuse risk. The only released artifacts are synthetic simulation code.
    \item[] Guidelines:
    \begin{itemize}
        \item The answer \answerNA{} means that the paper poses no such risks.
        \item Released models that have a high risk for misuse or dual-use should be released with necessary safeguards to allow for controlled use of the model, for example by requiring that users adhere to usage guidelines or restrictions to access the model or implementing safety filters. 
        \item Datasets that have been scraped from the Internet could pose safety risks. The authors should describe how they avoided releasing unsafe images.
        \item We recognize that providing effective safeguards is challenging, and many papers do not require this, but we encourage authors to take this into account and make a best faith effort.
    \end{itemize}

\item {\bf Licenses for existing assets}
    \item[] Question: Are the creators or original owners of assets (e.g., code, data, models), used in the paper, properly credited and are the license and terms of use explicitly mentioned and properly respected?
    \item[] Answer: \answerNA{}
    \item[] Justification: The paper does not use existing datasets, pretrained models, or third-party code assets as part of the research contribution. The work is theoretical and uses synthetic simulations.
    \item[] Guidelines:
    \begin{itemize}
        \item The answer \answerNA{} means that the paper does not use existing assets.
        \item The authors should cite the original paper that produced the code package or dataset.
        \item The authors should state which version of the asset is used and, if possible, include a URL.
        \item The name of the license (e.g., CC-BY 4.0) should be included for each asset.
        \item For scraped data from a particular source (e.g., website), the copyright and terms of service of that source should be provided.
        \item If assets are released, the license, copyright information, and terms of use in the package should be provided. For popular datasets, \url{paperswithcode.com/datasets} has curated licenses for some datasets. Their licensing guide can help determine the license of a dataset.
        \item For existing datasets that are re-packaged, both the original license and the license of the derived asset (if it has changed) should be provided.
        \item If this information is not available online, the authors are encouraged to reach out to the asset's creators.
    \end{itemize}

\item {\bf New assets}
    \item[] Question: Are new assets introduced in the paper well documented and is the documentation provided alongside the assets?
    \item[] Answer: \answerYes{}
    \item[] Justification: The paper releases synthetic simulation code and seeds. These assets are accompanied by documentation describing the simulation procedure, dependencies, commands needed to reproduce the figures, and an appropriate license.
    \item[] Guidelines:
    \begin{itemize}
        \item The answer \answerNA{} means that the paper does not release new assets.
        \item Researchers should communicate the details of the dataset\slash code\slash model as part of their submissions via structured templates. This includes details about training, license, limitations, etc. 
        \item The paper should discuss whether and how consent was obtained from people whose asset is used.
        \item At submission time, remember to anonymize your assets (if applicable). You can either create an anonymized URL or include an anonymized zip file.
    \end{itemize}

\item {\bf Crowdsourcing and research with human subjects}
    \item[] Question: For crowdsourcing experiments and research with human subjects, does the paper include the full text of instructions given to participants and screenshots, if applicable, as well as details about compensation (if any)? 
    \item[] Answer: \answerNA{}
    \item[] Justification: The paper does not involve crowdsourcing or research with human subjects.
    \item[] Guidelines:
    \begin{itemize}
        \item The answer \answerNA{} means that the paper does not involve crowdsourcing nor research with human subjects.
        \item Including this information in the supplemental material is fine, but if the main contribution of the paper involves human subjects, then as much detail as possible should be included in the main paper. 
        \item According to the NeurIPS Code of Ethics, workers involved in data collection, curation, or other labor should be paid at least the minimum wage in the country of the data collector. 
    \end{itemize}

\item {\bf Institutional review board (IRB) approvals or equivalent for research with human subjects}
    \item[] Question: Does the paper describe potential risks incurred by study participants, whether such risks were disclosed to the subjects, and whether Institutional Review Board (IRB) approvals (or an equivalent approval/review based on the requirements of your country or institution) were obtained?
    \item[] Answer: \answerNA{}
    \item[] Justification: The paper does not involve human subjects research, so IRB approval or equivalent review is not applicable.
    \item[] Guidelines:
    \begin{itemize}
        \item The answer \answerNA{} means that the paper does not involve crowdsourcing nor research with human subjects.
        \item Depending on the country in which research is conducted, IRB approval (or equivalent) may be required for any human subjects research. If you obtained IRB approval, you should clearly state this in the paper. 
        \item We recognize that the procedures for this may vary significantly between institutions and locations, and we expect authors to adhere to the NeurIPS Code of Ethics and the guidelines for their institution. 
        \item For initial submissions, do not include any information that would break anonymity (if applicable), such as the institution conducting the review.
    \end{itemize}

\item {\bf Declaration of LLM usage}
    \item[] Question: Does the paper describe the usage of LLMs if it is an important, original, or non-standard component of the core methods in this research? Note that if the LLM is used only for writing, editing, or formatting purposes and does \emph{not} impact the core methodology, scientific rigor, or originality of the research, declaration is not required.
    \item[] Answer: \answerNA{}
    \item[] Justification: The core methodology, theoretical results, simulations, and proofs do not use LLMs as an important, original, or non-standard research component. Any LLM use was limited to writing, editing, or formatting assistance and does not affect the scientific methodology.
    \item[] Guidelines:
    \begin{itemize}
        \item The answer \answerNA{} means that the core method development in this research does not involve LLMs as any important, original, or non-standard components.
        \item Please refer to our LLM policy in the NeurIPS handbook for what should or should not be described.
    \end{itemize}

\end{enumerate}

\end{document}